\documentclass{article}

\usepackage{microtype}
\usepackage{graphicx}
\usepackage{subcaption}
\usepackage{booktabs} 

\usepackage{hyperref}

\usepackage[accepted]{icml2026}

\usepackage{amsmath}
\usepackage{amssymb}
\usepackage{mathtools}
\usepackage{amsthm}
\usepackage{booktabs}
\usepackage{multirow}
\usepackage{array}
\usepackage{graphicx} 

\usepackage[capitalize,noabbrev]{cleveref}

\theoremstyle{plain}
\newtheorem{theorem}{Theorem}[section]

\newtheorem{lemma}[theorem]{Lemma}

\theoremstyle{definition}

\theoremstyle{remark}

\usepackage[textsize=tiny]{todonotes}

\icmltitlerunning{FAST-AR: Fast Autoregressive Video Diffusion and World Models with
Temporal Cache Compression and Sparse Attention}

\begin{document}

\twocolumn[
  \icmltitle{FAST-AR: Fast Autoregressive Video Diffusion and World Models with Temporal Cache Compression and Sparse Attention}



  \icmlsetsymbol{equal}{*}

  \begin{icmlauthorlist}
    \icmlauthor{Dvir Samuel}{originai,biu}
    \icmlauthor{Issar Tzachor}{originai}
    \icmlauthor{Matan Levy}{huji}
    \icmlauthor{Michael Green}{originai}
    \icmlauthor{Gal Chechik}{biu,nvidia}
    \icmlauthor{Rami Ben-Ari}{originai}
  \end{icmlauthorlist}

  \icmlaffiliation{originai}{OriginAI, Tel-Aviv, Israel}
  \icmlaffiliation{biu}{Bar-Ilan University, Ramat-Gan, Israel}
  \icmlaffiliation{huji}{The Hebrew University of Jerusalem, Jerusalem, Israel}
  \icmlaffiliation{nvidia}{NVIDIA, Tel-Aviv, Israel}

  \icmlcorrespondingauthor{Dvir Samuel}{dvirsamuel@gmail.com}

  \icmlkeywords{Machine Learning, ICML}

  \vskip 0.3in
  
]



\printAffiliationsAndNotice{}  
\begin{figure*}[t!]
	\centering
	\includegraphics[width=1\linewidth]{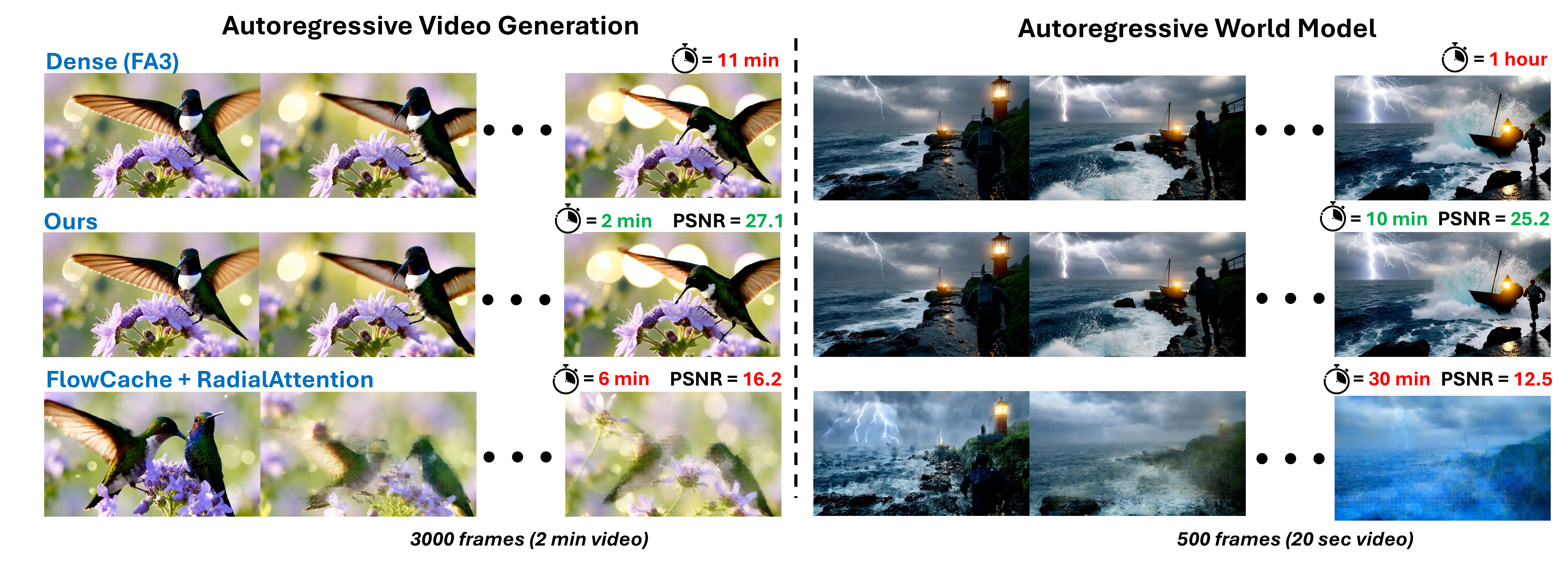}
	\caption{Our method substantially accelerates \emph{pre-trained} autoregressive video diffusion models and autoregressive world models while maintaining high visual quality, by introducing a new KV-cache compression with self- and cross-attention sparsification. On a single H100 GPU, it achieves $5\times$--$10\times$ speedups for multi-minute video generation, without further training/optimization, and keeps peak GPU memory nearly constant over long rollouts.}

\label{fig:fig1}
\end{figure*}

\begin{abstract}
Autoregressive video diffusion models enable \emph{streaming} generation, opening the door to long-form synthesis, video world models, and interactive neural game engines. However, their core attention layers become a major bottleneck at inference time: as generation progresses, the KV cache grows, causing both increasing latency and escalating GPU memory, which in turn restricts usable temporal context and harms long-range consistency. In this work, we study redundancy in autoregressive video diffusion and identify three persistent sources: near-duplicate cached keys across frames, slowly evolving (largely semantic) queries/keys that make many attention computations redundant, and cross-attention over long prompts where only a small subset of tokens matters per frame.
Building on these observations, we propose a unified, training-free attention framework (\textbf{FAST-AR}) for \textbf{FAST}-\textbf{A}uto\textbf{R}egressive diffusion, consisting of three components: \textbf{TempCache} compresses the KV cache via temporal correspondence to bound cache growth; \textbf{AnnCA} accelerates cross-attention by selecting frame-relevant prompt tokens using fast approximate nearest neighbor (ANN) matching; and \textbf{AnnSA} sparsifies self-attention by restricting each query to semantically matched keys, also using a lightweight ANN. Together, these modules reduce attention, compute, and memory and are compatible with existing autoregressive diffusion backbones and world models. Experiments demonstrate up to $\times 5$--$\times 10$ end-to-end speedups while preserving near-identical visual quality and, crucially, maintaining stable throughput and nearly constant peak GPU memory usage over long rollouts, where prior methods progressively slow down and suffer from increasing memory usage. \href{https://dvirsamuel.github.io/fast-auto-regressive-video/}{Project Page}
\end{abstract}

\section{Introduction}

Video diffusion models~\cite{HaCohen2024LTXVideoRV, Wang2025WanOA, kong2024hunyuanvideo} have achieved strong results in offline video synthesis, where all frames in a short clip are generated jointly, producing high visual fidelity and temporally coherent motion. Recently, \emph{autoregressive} video diffusion models \cite{yin2025causvid, huang2025selfforcing} have emerged to enable streaming generation: frames are produced sequentially in an online manner and can be consumed immediately. This transition from offline to online generation unlocks new applications, including long-form video generation~\cite{liu2025rolling}, controlled video world models~\cite{Agarwal2025CosmosWF,gao2025longvie2}, and neural game engines~\cite{gao2025longvie2,hunyuanworld2025tencent}.

Despite these new capabilities,  autoregressive video diffusion models expose a critical bottleneck in their attention mechanisms, leading to two fundamental challenges:
\textbf{(a) Generation speed.} 3D spatio-temporal attention scales with the number of cached keys. As the KV cache grows with each generated frame, the per-step attention cost (and thus latency) increases over time, making long or unbounded generation progressively slower.
\textbf{(b) Memory efficiency.} The expanding KV cache also creates substantial memory overhead, limiting how much temporal context can fit on GPU. This often forces short context windows \cite{liu2025rolling}, which in turn harms long-range temporal consistency.

Despite KV-cache efficiency being central to autoregressive video generation, it remains relatively underexplored: KV-cache compression techniques from NLP (for LLMs)  can not be applied directly~\cite{FlowCache}, and recent methods~\cite{TeaCache, FlowCache} offer only modest speedups, sometimes with quality degradation. 
In parallel, most progress on accelerating 3D spatio-temporal attention targets \emph{offline} video diffusion~\cite{SVG1, SVG2, li2025radial}. Yet these gains do not reliably transfer to online autoregressive models, where the same methods frequently degrade quality and deliver limited, or even negative, speedups due to unfavorable memory and cache behavior.
Following this, we therefore analyze whether autoregressive video diffusion contains exploitable redundancy, and find three consistent sources that are not directly addressed by existing sparse attention designs: (i) many cached keys are near-duplicates across frames, enabling aggressive KV merging; (ii) Q and K evolve slowly and are largely semantic, making many query--key dot products (attention score computations) redundant; and (iii) cross-attention wastes compute on long prompts although only a few tokens matter per frame.

Motivated by these observations, we introduce a unified attention framework, \textbf{FAST-AR} (\textbf{Fast} \textbf{A}uto\textbf{R}egressive diffusion), consisting of three components: \textbf{TempCache} compresses the KV cache using temporal correspondence; \textbf{AnnCA} prunes prompt tokens per frame, in cross-attention layers, using fast Approximate Nearest Neighbors (ANN) matching; and \textbf{AnnSA} sparsifies self-attention by restricting each query to semantically matched keys, also using lightweight ANN. To the best of our knowledge, this is the first use of ANN-based attention in autoregressive diffusion models in a fully training-free manner, requiring no retraining or fine-tuning.
Experiments on autoregressive video diffusion and video world models show that our approach yields up to \textbf{$\times 5$--$\times 10$} end-to-end speedup (Figure~\ref{fig:fig1}), while maintaining nearly constant GPU memory over long rollouts, whereas baselines accumulate cache with increasing memory, and degrade video quality.

\section{Related Work}

\noindent\textbf{Autoregressive video diffusion models and video world models.}
Modern video diffusion~\cite{Wang2025WanOA, hunyuanworld2025tencent,HaCohen2024LTXVideoRV} models often adopt diffusion transformers (DiTs), which scale well but make video generation expensive due to long spatiotemporal token sequences and repeated denoising steps. To extend diffusion to long horizons, recent work~\cite{yin2025causvid,huang2025selfforcing} studies \emph{autoregressive} (chunked) video diffusion, repeatedly denoising the next segment conditioned on previously generated history. \textsc{Self Forcing}~\cite{huang2025selfforcing} reduces train--test mismatch by rolling out autoregressive generations with KV caching during training, optimizing predictions on self-generated context with a video-level objective. \textsc{Rolling Forcing}~\cite{liu2025rolling} targets multi-minute streaming by relaxing strict causality using joint denoising with increasing noise levels and by retaining early KV states as an attention sink to reduce error accumulation. In parallel, \textsc{LongVie2}~\cite{gao2025longvie2} frames long video generation as world modeling and proposes a staged training recipe to improve controllability and long-term consistency. Complementary to generation, \textsc{DiffTrack}~\cite{difftrack} studies temporal correspondence signals in video DiTs, showing that query--key similarity at specific layers/timesteps enables matching across frames for zero-shot point tracking.

\noindent\textbf{KV compression and caching for autoregressive video diffusion.}
Since diffusion inference evaluates the denoiser repeatedly across timesteps, training-free caching can reuse intermediate computations when changes are small. \textsc{TeaCache}~\cite{TeaCache} uses a lightweight timestep-aware estimator to decide when cached outputs can be reused, providing speedups with limited quality loss. For long-horizon autoregressive generation, \textsc{FlowCache}~\cite{FlowCache} proposes chunkwise recomputation policies and importance-based KV cache compression to bound memory while accelerating ultra-long rollouts. Unlike \cite{TeaCache} and \cite{FlowCache}, we directly compress the \emph{attention} KV cache by leveraging temporal correspondence across frames, merging redundant keys that track the same content over time. This training-free, correspondence-driven merging bounds KV growth and stabilizes both latency and peak memory over long rollouts.

\noindent\textbf{Self-attention sparsification for diffusion models.}
A major bottleneck in video diffusion transformers is the quadratic cost of full 3D self-attention over space--time tokens.
\textsc{SVG}~\cite{SVG1} accelerates \emph{offline} video diffusion by exploiting structured head sparsity, while \textsc{SVG2}~\cite{SVG2} improves token selection using semantic clustering and token permutation to reduce sparse-kernel inefficiency. \textsc{Radial Attention}~\cite{li2025radial} uses an energy-decay prior to build static spatiotemporal masks with sub-quadratic cost. These methods target full-video generation; in contrast, self-attention sparsification tailored to autoregressive/chunked diffusion, with causality, KV caching, and long rollouts, remains largely unexplored.

\noindent\textbf{Approximate nearest neighbor search.}
Approximate nearest neighbor (ANN) search accelerates high-dimensional similarity retrieval by trading exactness for speed, often using an offline preprocessing phase, e.g., graph indices such as HNSW~\cite{Malkov2016EfficientAR}, hashing-based methods, or clustering-based indices (e.g., inverted files).
Locality-sensitive hashing (LSH)~\cite{lsh} offers a lightweight alternative with principled hash families and multi-table bucket probing, avoiding exhaustive search while providing approximation guarantees. Quantization-based ANN~\cite{quantization} compresses vectors into compact codes; product quantization (PQ) quantizes low-dimensional subspaces and enables fast distance estimation using lookup tables with asymmetric distance computation. Reformer~\cite{reformer} integrates LSH into the Transformer and trains around LSH-based attention routing. In this paper, we use LSH/quantization \emph{training-free at inference time} to approximate nearest-neighbor matching without modifying or retraining the underlying model.

\section{Auto-regressive Video Diffusion Models}

We begin by reviewing the attention formulation used in autoregressive video diffusion models. Given queries $Q \in \mathbb{R}^{N_q \times d}$, keys $K \in \mathbb{R}^{N_k \times d}$, and values $V \in \mathbb{R}^{N_k \times d_v}$, attention is computed as
\begin{equation}
O = \mathrm{softmax}\!\left(\frac{Q \hat{K}^\top}{\sqrt{d}}\right)\hat{V},
\end{equation}
where $\hat{K}$ and $\hat{V}$ denote the KV cache, formed by concatenating keys and values from the current frame with those from all previously generated frames. In cross-attention layers, $\hat{K}$ and $\hat{V}$ correspond to projected prompt tokens of the current frame.
As generation proceeds, the KV cache grows with the number of generated frames, so the per-step attention cost scales linearly with cache length (and the cumulative attention work over a $T$-frame rollout can grow as $\mathcal{O}(T^2)$), leading to increasing inference latency and memory usage.

\begin{figure}[t!]
	\centering
	\includegraphics[width=0.85\linewidth]{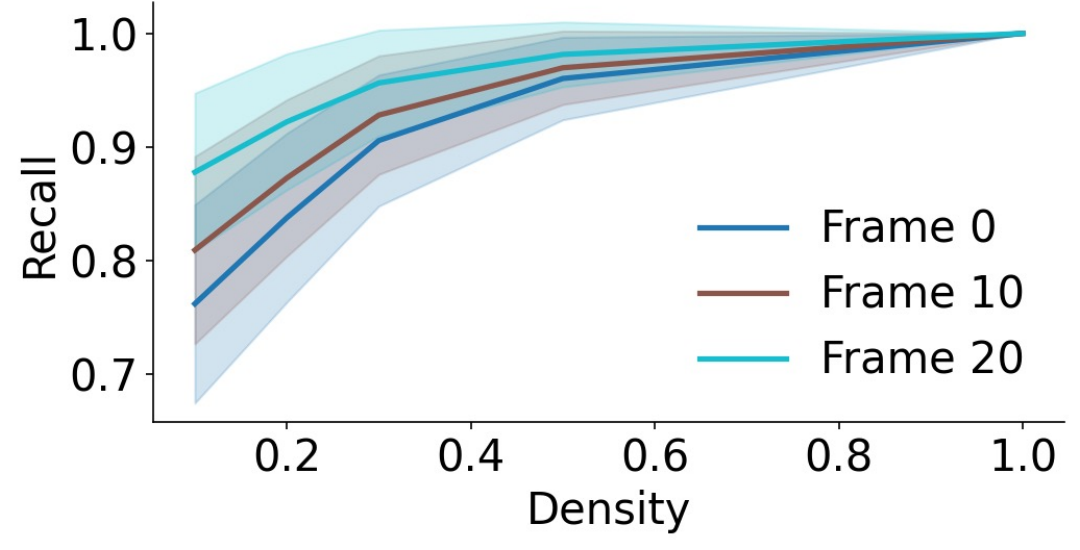}
	\caption{\textbf{Attention sparsity in autoregressive video diffusion.}
Attention recall vs.\ density on Rolling-Forcing~\cite{liu2025rolling}, averaged over transformer blocks (shaded: std). Density is induced by keeping only the highest-attention entries. This achieves high recall, e.g., $\approx$85\% at 30\% density, indicating substantial sparsity.
}

	\label{fig:recall}
\end{figure}

\begin{figure*}[t!]
	\centering
	\includegraphics[width=0.9\linewidth]{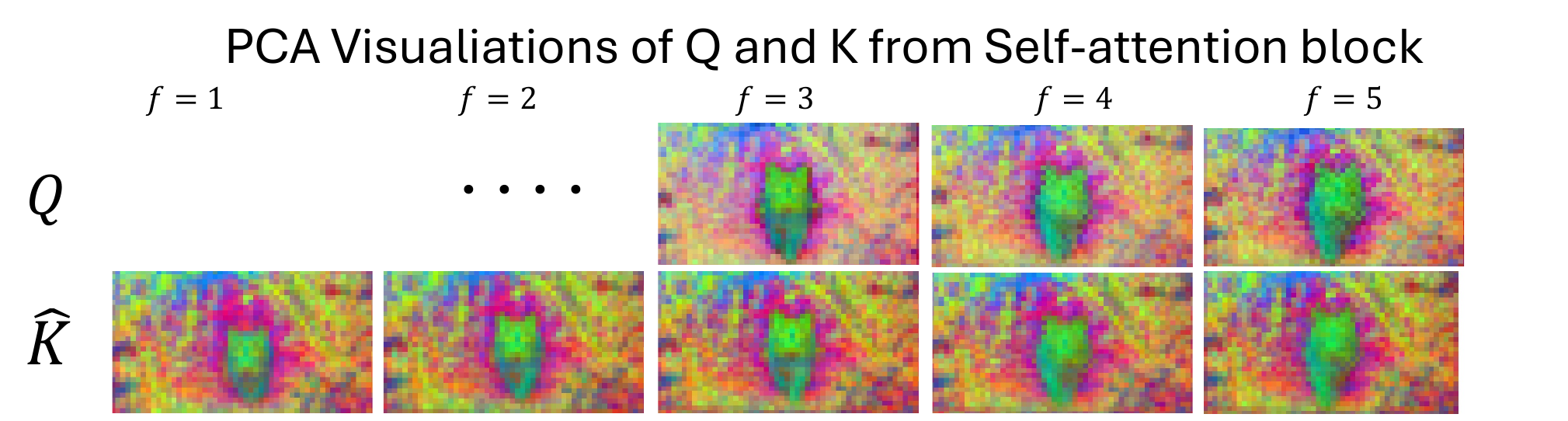}
	\caption{\textbf{Semantic structure and temporal redundancy in self-attention features.} PCA of self-attention queries $Q$ and cached keys $\hat{K}$ across frames (similar colors denote nearby embeddings) for a generated video of a cat walking toward the camera. The features exhibit semantic clustering (foreground vs.\ background) and strong key repetition across frames, motivating KV-cache compression.}

	\label{fig:motivation}
\end{figure*}

\section{Motivation}
\label{sec:motivation}

\noindent\textbf{Does attention in autoregressive video diffusion exhibit sparsity?} 
Prior work~\cite{SVG2} has shown that self-attention in \emph{offline} video diffusion models is naturally sparse: only a small fraction of attention computations substantially contribute to the final output. However, it remains unclear whether similar sparsity exists in the \emph{autoregressive} setting, where frames are generated sequentially and the attention context grows over time.

To investigate this, we generate 100 videos with Rolling-Forcing~\cite{liu2025rolling} and quantify attention sparsity via \emph{attention recall}, defined as the fraction of dense attention mass preserved after retaining only a top fraction of the largest attention entries. Specifically, for each transformer block, we keep the highest-attention entries, compute recall relative to dense attention, and report the mean and standard deviation across blocks. As shown in Figure~\ref{fig:recall}, attention in autoregressive video diffusion models is highly sparse: retaining only $30\%$ of the computations already preserves more than $85\%$ of the attention mass, indicating substantial headroom for improving efficiency.

\noindent\textbf{Where does sparsity arise during generation?}
To understand the source of this sparsity, we analyze the internal representations of attention modules across different generation timesteps. Our objective is to identify redundant computations that can be reduced or removed without degrading generation quality.

We first focus on the self-attention layers and examine the query ($Q$) and key ($\hat{K}$) features throughout generation. 
Figure~\ref{fig:motivation} illustrates Principal Component Analysis (PCA) applied to Q and K features. Tokens with similar colors correspond to features that are close in the original embedding space.  The visualization reveals two key properties. First, both Q and K are predominantly semantic: foreground tokens (e.g., the cat) cluster together and attend to corresponding semantic regions, while background tokens form separate clusters, consistent with prior findings that attention features encode object-level structure~\cite{OmnimatteZero, story2board}. Second, a large fraction of key features repeat across frames, with many keys from earlier frames reappearing later. This repetition persists across diffusion timesteps and in almost all transformer layers, suggesting that aggressive KV-cache compression is feasible.

\begin{figure*}[t!]
	\centering
	\includegraphics[width=0.9\linewidth]{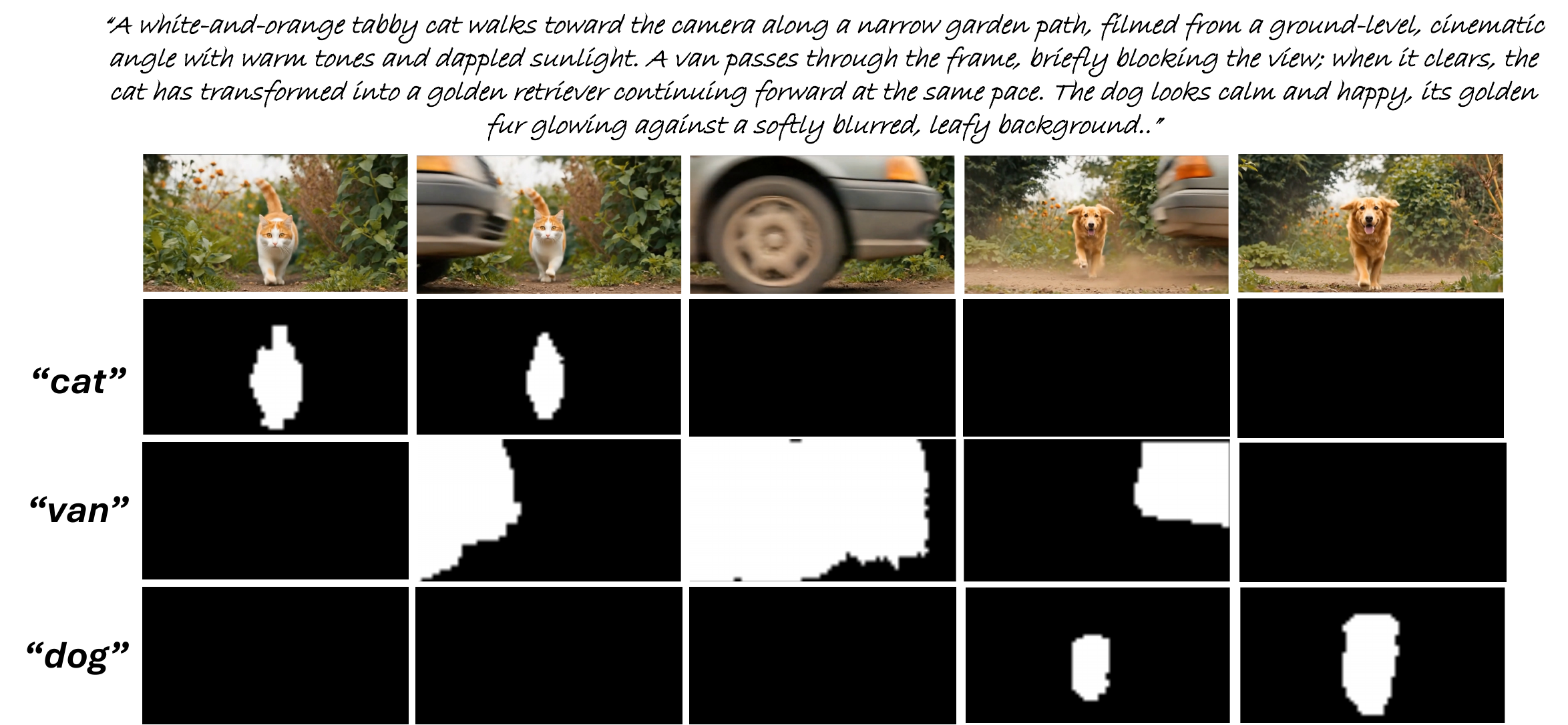}
	\caption{\textbf{Cross-attention is frame-selective.} Input prompt (top) with Per-token cross-attention maps (bottom) for ``cat'', ``van'', and ``dog'' across frames. Attention concentrates on the tokens relevant to the current content (cat early, van during occlusion, dog after transformation), suggesting that pruning irrelevant prompt tokens per frame can reduce cross-attention compute.}

	\label{fig:ca_vis}
\end{figure*}

We next examine the cross-attention layers by analyzing the attention maps between textual prompt tokens and video latents. As depicted in Figure \ref{fig:ca_vis}, autoregressive video diffusion models are typically conditioned on long, detailed prompts~\cite{huang2025selfforcing, liu2025rolling} that describe objects, actions, and events spanning the entire video. As a result, cross-attention incurs substantial computational overhead: at every layer and generation timestep, each query attends to all prompt tokens, even though only a small subset is relevant for synthesizing the current frame.
Cross-attention is highly selective: for each frame, most attention mass concentrates on a small subset of prompt tokens, while the rest contribute little. This observation suggests that dynamically selecting a frame-specific subset of relevant prompt tokens can dramatically reduce cross-attention computation without sacrificing quality.

\section{Method}

Our goal is to accelerate autoregressive video diffusion by reducing the compute and memory cost of attention without degrading generation quality. We do so with three components: (i) \textbf{TempCache} for temporal KV-cache compression, (ii) \textbf{AnnSA} for sparse self-attention, and (iii) \textbf{AnnCA} for sparse cross-attention; all three use fast approximate nearest-neighbor matching to select a small set of candidate tokens before computing attention.

\subsection{Attention as Approximate Nearest Neighbor Search}
\label{sec:ann}

Directly computing attention scores via $Q\hat{K}^\top$ is expensive. Efficient kernels such as FlashAttention~\cite{flashattenion3} and FlashInfer~\cite{ye2025flashinfer} avoid explicitly materializing this matrix, but the overall cost still scales with the number of cached keys per step, and thus becomes prohibitive over long autoregressive generations.

Motivated by our empirical findings (Section~\ref{sec:motivation}), which show that only a small fraction of keys meaningfully contribute to each query, we reinterpret attention as an \emph{approximate nearest neighbor} (ANN) problem. For each query, the goal is to identify the keys with the largest dot products and compute attention only over them.

Formally, instead of attending to all cached keys $\hat{K}$ (and values $\hat{V}$), we first retrieve a small candidate set of key indices
\[
\mathcal{N}(q) \subseteq \{1,\dots,|\hat{K}|\},
\]
and denote by $\hat{K}_{\mathcal{N}(q)} \in \mathbb{R}^{|\mathcal{N}(q)| \times d}$ and
$\hat{V}_{\mathcal{N}(q)} \in \mathbb{R}^{|\mathcal{N}(q)| \times d_v}$
the corresponding rows of the KV cache. We then approximate attention as
\[
\mathrm{Attn}(q) \approx \mathrm{Softmax}\!\left(\frac{q \hat{K}_{\mathcal{N}(q)}^\top}{\sqrt{d}}\right)\hat{V}_{\mathcal{N}(q)}.
\]

Since ANN search must be performed repeatedly across layers and diffusion timesteps, it must incur minimal preprocessing overhead. We consider two lightweight approximate nearest neighbor strategies:

\noindent\textbf{Locality-Sensitive Hashing (LSH).}
We project queries and keys into a shared low-dimensional random subspace and bucket them with locality-sensitive hashing~\cite{lsh}, so that a query only compares against keys that land in the same bucket(s).

\noindent\textbf{Quantized Similarity Search.}
Alternatively, we quantize~\cite{quantization} queries and keys to a low-bit representation and perform nearest-neighbor search directly in the quantized space.

Both approaches reduce matching cost by limiting dot products to a small candidate set, lowering compute and memory bandwidth while retaining sufficient attention recall, and do so without offline preprocessing, enabling fast inference-time matching.
After candidates are selected, attention is computed with sparse attention kernels~\cite{ye2025flashinfer}. 

Having established a fast approximation for attention computation, we next focus on reducing the size of the KV cache itself.

\subsection{KV Cache Compression using Temporal Correspondence}

\begin{figure}[h!]
	\centering
	\includegraphics[width=\linewidth]{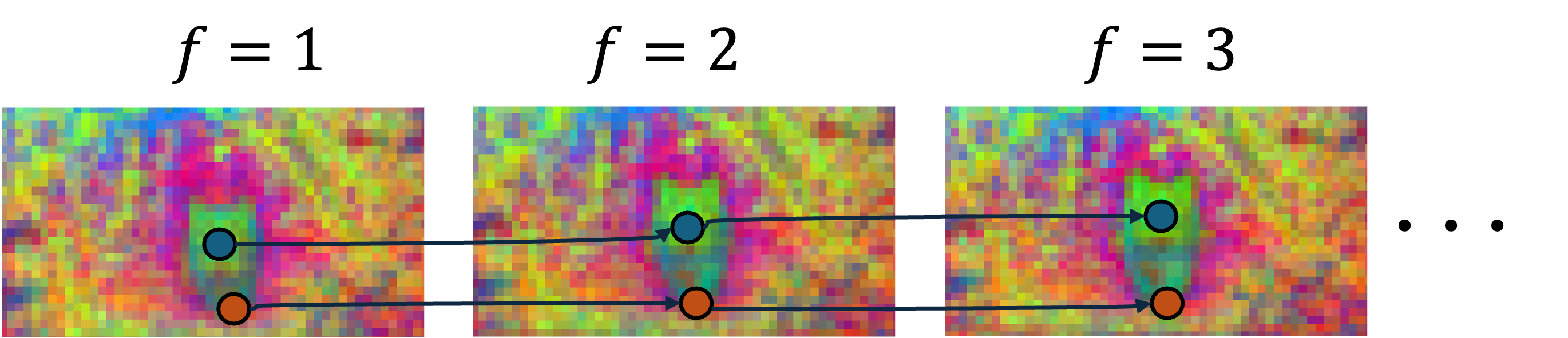}
	\caption{\textbf{Temporal correspondence for KV compression.} Many key features persist across frames and can be matched using temporal correspondence (colored dots/arrows). Correspondences are recovered by selecting, for each current-frame query, the most related key in a previous frame~\cite{difftrack}}

	\label{fig:temoral_correspondance}
\end{figure}

In Section~\ref{sec:motivation}, we showed that key features exhibit substantial temporal redundancy: many keys from earlier frames reappear nearly unchanged in later frames. This observation suggests that the KV cache can be dramatically compressed by identifying temporally corresponding tokens and retaining only representative keys.
Figure~\ref{fig:temoral_correspondance} illustrates this phenomenon: tokens corresponding to the same semantic regions persist across frames and follow consistent trajectories in feature space.

\noindent\textbf{Finding temporal correspondences.}
To identify corresponding tokens across frames, we build on the observation from DiffTrack~\cite{difftrack} that temporal correspondence can be recovered directly from attention. Specifically, for a query token in the current frame, its corresponding token in a previous frame is the key with the highest attention score.
To avoid explicitly computing $QK^\top$, we instead use our ANN machinery (section \ref{sec:ann}) to retrieve the top-1 nearest neighbor key for each query. This provides an efficient estimate of temporal correspondence during generation.

Once correspondences are established, we group keys across frames that correspond to the same semantic token. For each group, we retain only the recent representative key in the KV cache.
However, naively removing keys raises three concerns: (1) How to apply this compression without explicitly forming the attention score matrix $QK^{\top}$, since current kernels do not materialize it, (2) how to correctly aggregate values $\hat{V}$, and (3) whether attention computed on the compressed cache remains accurate. We address all concerns with the following lemma, which shows that attention over duplicate keys can be computed \emph{exactly} using a grouped representation.

\begin{lemma}[Redundancy-free attention]
\label{lem:redundancy_free_attention}
Let $q \in \mathbb{R}^{d_k}$ be a query, and let
$K = (k_1, \dots, k_n)^\top$ and $V = (v_1, \dots, v_n)^\top$
denote keys and values. Suppose the indices are partitioned into groups
$\{G_t\}_{t=1}^g$ such that all keys within a group are identical.
Then the standard attention over $(K,V)$ is \emph{exactly} equal to an
attention computed over the $g$ group representatives: for group $t$,
the logit is shifted by $+\log m_t$, where $m_t = |G_t|$, and the
associated value is the mean of $\{v_i\}_{i\in G_t}$.
\end{lemma}

The full statement and proof are given in the appendix.

In other words, the lemma states that merging duplicate keys incurs \emph{no} approximation error: when redundancy is exact, attention can be computed exactly from a compressed KV cache. In practice, temporal features are only approximately redundant, so keys are merged using a similarity threshold, yielding a controlled approximation. In the worst case, when no redundancy is detected, the method reduces to standard dense attention. Importantly, the procedure is kernel-agnostic: it modifies only the \emph{inputs} to attention (the cached $K,V$) and requires no changes to the attention kernel itself. We refer to this approach as \textbf{TempCache} (\textbf{Temp}oral correspondence KV \textbf{Cache} compression).

\subsection{Cross-Attention Redundancy}
\label{sparse_ca}

Autoregressive video diffusion models are often conditioned on long textual prompts (Figure~\ref{fig:ca_vis}), making cross-attention expensive: at each layer and timestep, queries attend to all prompt tokens, although only a few are relevant for the current frame.
We reduce this overhead by selecting only the prompt tokens that are relevant to the current frame. Crucially, we avoid computing dense attention maps between all latent queries and all prompt keys. Instead, we leverage approximate nearest neighbor (ANN) search to identify relevant prompt tokens efficiently.
Concretely, we project both the latent queries of the current frame and the prompt keys into a shared LSH- or quantized embedding space. For each prompt token, we check whether it has at least one neighboring latent query in its bucket. Prompt tokens that do not share any bucket with any current-frame query are considered irrelevant and excluded from cross-attention computation for that frame. We call this approach \textbf{AnnCA} (\textbf{A}pproximate \textbf{N}earest-\textbf{N}eighbor \textbf{C}ross-\textbf{A}ttention).

\subsection{Self-Attention Redundancy}

Self-attention in video diffusion models exhibits strong structural patterns: tokens tend to attend primarily to other tokens that are semantically related and spatially or temporally proximate (Figure~\ref{fig:motivation}). We exploit this structure to sparsify self-attention by reusing the semantic buckets discovered during cross-attention pruning (Section~\ref{sparse_ca}), effectively transferring prompt-induced semantic grouping into self-attention. Each token’s query is assigned to one or more buckets, and during self-attention, we restrict it to attend only to keys within the same bucket(s), enforcing semantic locality. We compute this bucketed attention efficiently using block-sparse attention kernels (e.g., FlashInfer~\cite{ye2025flashinfer}. We call this approach \textbf{AnnSA} (\textbf{A}pproximate \textbf{N}earest \textbf{N}eighbor \textbf{S}elf \textbf{A}ttention).

\section{Experiments}

\subsection{Setup}

\noindent\textbf{Models.}
We evaluate on state-of-the-art open-source autoregressive video diffusion and world models:
\textbf{(i) Rolling-Forcing}~\cite{liu2025rolling}, a real-time long-horizon autoregressive video diffusion method for multi-minute generation and
\textbf{(ii) LongVie2}~\cite{gao2025longvie2}, an autoregressive transformer-based world model that frames long video generation as world modeling and introduces a staged training recipe (multimodal guidance and history alignment) to improve controllability and long-term consistency.

\noindent\textbf{Metrics.}
Following prior work on efficient video diffusion attention~\cite{SVG1, SVG2}, we evaluate fidelity to the dense-attention baseline using PSNR, SSIM, and LPIPS computed between videos generated with our approach and videos generated with Dense FlashAttention-3 (FA3) under matched prompts and random seeds.
We additionally report LongVBench~\cite{SVG2}/LongVGenBench~\cite{gao2025longvie2} scores as measures of perceptual video quality.
To quantify attention efficiency, we report the attention density (fraction of executed query--key interactions relative to dense attention) and attention recall (fraction of attention mass preserved by the sparse pattern).
Finally, we report per-component and end-to-end \emph{generation acceleration} as the speedup over Dense FA3.

\begin{table*}[t]
\centering
\small
\setlength{\tabcolsep}{6.0pt}
\renewcommand{\arraystretch}{1.05}
\caption{\textbf{Rolling-Forcing (LongVBench).} Our method achieves the best quality--efficiency trade-off across all settings. TempCache (LSH/Quant) compresses the KV cache with the highest recall while matching dense attention quality. Combining TempCache with ANN-based SA/CA yields up to \textbf{$\times$10.7--$\times$10.8} end-to-end speedup.}

\label{tab:rf_longvbench_compact}
\resizebox{\textwidth}{!}{%
\begin{tabular}{llcccccccc}
\toprule
\textbf{RollingForcing} & \textbf{Method} &
\textbf{PSNR}$\uparrow$ & \textbf{SSIM}$\uparrow$ & \textbf{LPIPS}$\downarrow$ &
\textbf{VBench}$\uparrow$ & 
\textbf{\shortstack{Min\\Density $\downarrow$}} &
\textbf{\shortstack{Max\\Recall $\uparrow$}} &
\textbf{\shortstack{Component \\ Speedup $\uparrow$}} &
\textbf{\shortstack{E2E \\ Speedup $\uparrow$}} \\
\midrule
& Dense (FlashAttention 3) & -- & -- & -- & 84.08 & 100\% & 100\% & $\times$1.0 & $\times$1.0 \\
\midrule
\multicolumn{10}{l}{\textbf{KV compression}} \\
& TeaCache & 16.12 & 0.315 & 0.523 & 84.11 & 93.2\% & 84.6\% & $\times$1.1 & -- \\
& FlowCache & 22.15 & 0.634 & 0.222 & 84.15 & 82.9\% & 86.9\% & $\times$2.3 & -- \\
& \textbf{TempCache-LSH (ours)} & \textbf{24.13} & \textbf{0.651} & \textbf{0.149} & \textbf{84.17} & \textbf{16.8\%} & \textbf{90.2\%} & \textbf{$\times$6.8} & \textbf{$\times$3.0} \\
& \textbf{TempCache-Quant (ours)} & \textbf{24.26} & \textbf{0.653} & \textbf{0.143} & \textbf{84.19} & \textbf{16.2\%} & \textbf{91.4\%} & \textbf{$\times$6.9} & \textbf{$\times$3.2} \\
\midrule
\multicolumn{10}{l}{\textbf{Sparse Self-Attention}} \\
& SVG1 & 14.22 & 0.201 & 0.744 & 33.15 & 76.1\% & 30.2\% & $\times$0.1 & -- \\
& SVG2 & 14.29 & 0.226 & 0.736 & 34.66 & 68.8\% & 38.9\% & $\times$0.2 & -- \\
& RadialAttn & 16.87 & 0.289 & 0.702 & 61.51 & 81.6\% & 58.3\% & $\times$2.8 & -- \\
& \textbf{AnnSA-LSH (ours)} & \textbf{25.73} & \textbf{0.688} & \textbf{0.142} & \textbf{83.25} & \textbf{27.6\%} & \textbf{92.4\%} & \textbf{$\times$5.1} & \textbf{$\times$2.7} \\
& \textbf{AnnSA-Quant (ours)} & \textbf{25.77} & \textbf{0.689} & \textbf{0.141} & \textbf{83.29} & \textbf{28.0\%} & \textbf{92.6\%} & \textbf{$\times$5.2} & \textbf{$\times$2.8} \\
\midrule
\multicolumn{10}{l}{\textbf{Sparse Cross-Attention}} \\
& \textbf{AnnCA-LSH (ours)} & \textbf{25.68} & \textbf{0.679} & \textbf{0.155} & \textbf{83.23} & \textbf{33.1\%} & \textbf{94.2\%} & \textbf{$\times$2.2} & \textbf{$\times$1.2} \\
& \textbf{AnnCA-Quant (ours)} & \textbf{24.11} & \textbf{0.646} & \textbf{0.148} & \textbf{82.89} & \textbf{29.5\%} & \textbf{91.1\%} & \textbf{$\times$2.3} & \textbf{$\times$1.2} \\
\midrule
\multicolumn{10}{l}{\textbf{Full}} \\
& FlowCache+RadialAttn & 16.98 & 0.294 & 0.687 & 45.15 & -- & -- & -- & $\times$4.4 \\
& \textbf{All Ours-LSH} & \textbf{25.71} & \textbf{0.681} & \textbf{0.147} & \textbf{84.02} & -- & -- & -- & \textbf{$\times$10.7}\\
& \textbf{All Ours-Quant} & \textbf{25.73} & \textbf{0.678} & \textbf{0.147} & \textbf{83.99} & -- & -- & -- & \textbf{$\times$10.8} \\
\bottomrule
\end{tabular}%
}
\end{table*}

\begin{figure*}[t!]
	\centering
	\includegraphics[width=0.9\linewidth]{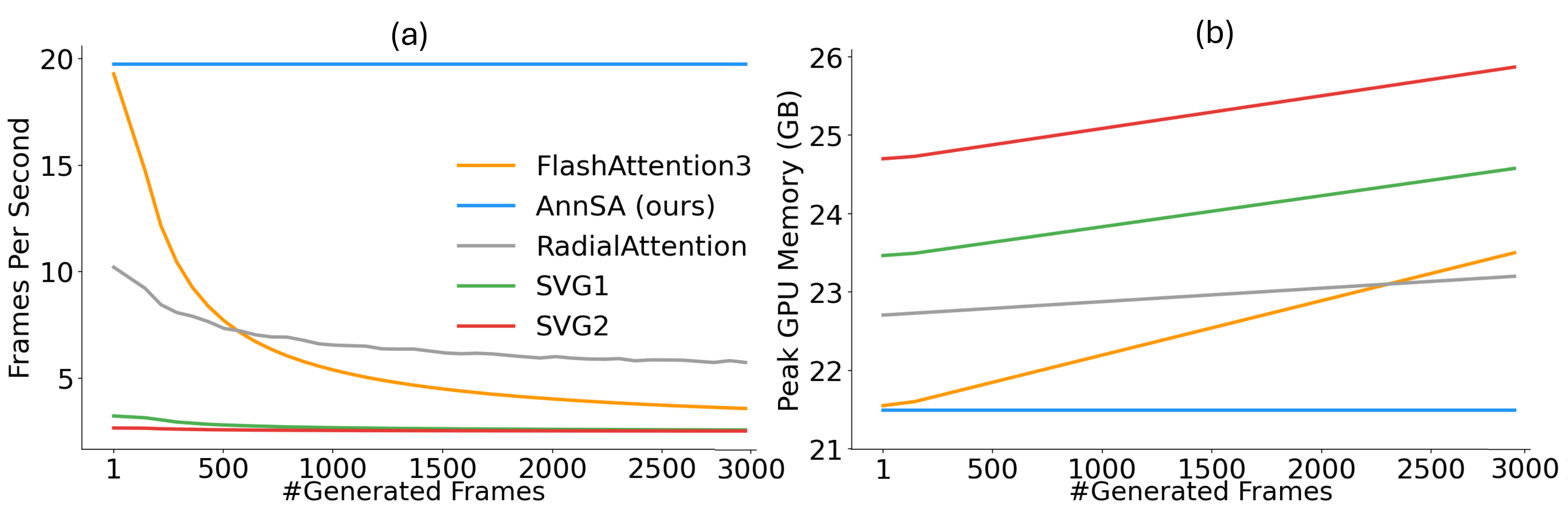}
	\caption{\textbf{Scaling with generation length in Rolling-Forcing.} 
(a) \emph{Throughput.} As context grows, Dense FA3 slows substantially, and prior sparsification baselines (SVG1/2, RadialAttention) fail to maintain throughput due to heavy per-block preprocessing repeated across blocks, timesteps, and frames. Our method sustains nearly constant FPS over a 3K-frame rollout, keeping attention cost effectively independent of cache length. 
(b) \emph{Peak memory.} FA3 and baselines exhibit increasing GPU memory as the KV cache expands, while our memory remains flat, consistent with a bounded cache.}

	\label{fig:time_memory_scaling}
\end{figure*}

\begin{figure*}[t!]
	\centering
	\includegraphics[width=0.9\linewidth]{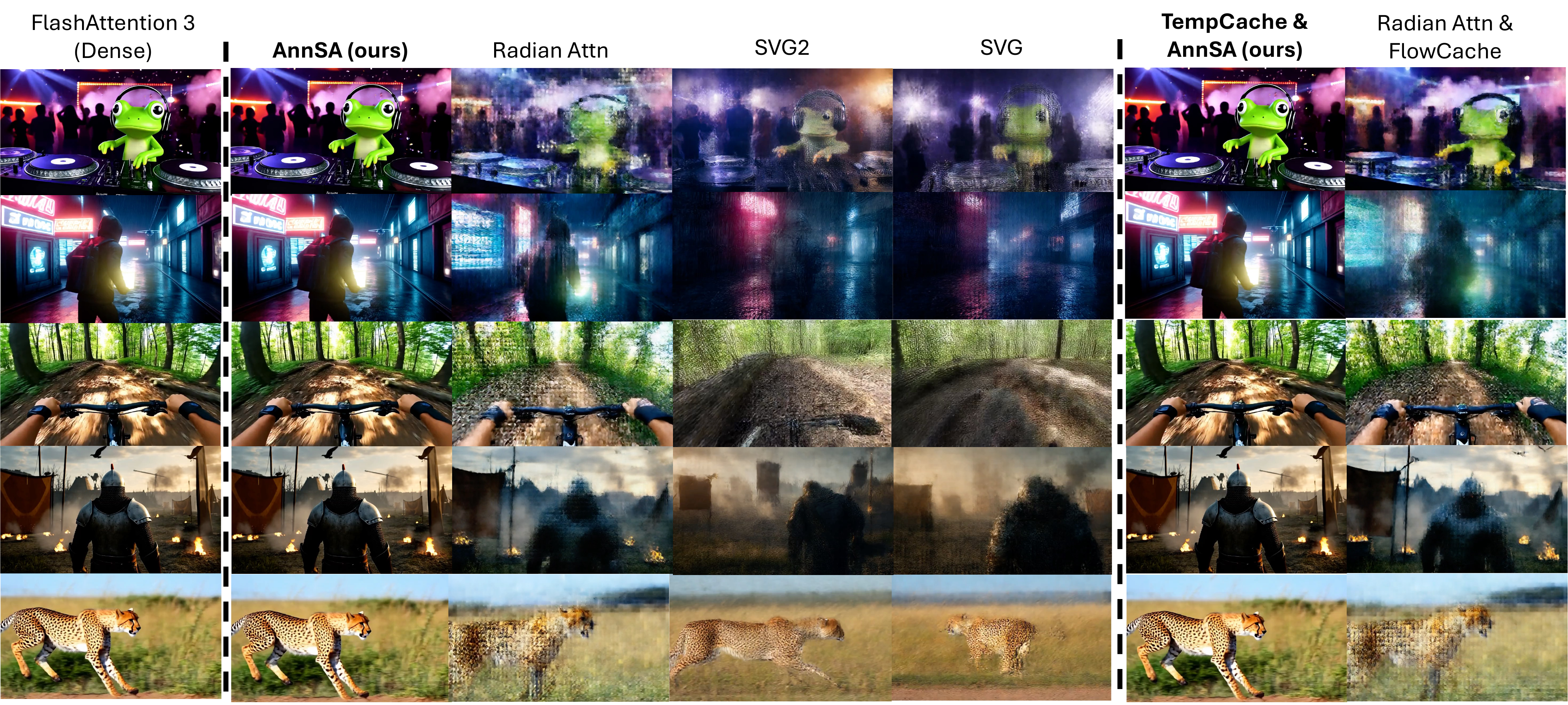}
	\caption{\textbf{Qualitative results on long-video generation with Rolling-Forcing.} Our approach preserves the visual fidelity and temporal consistency of Dense FA3 across diverse prompts, while current sparsification baselines (SVG1/2) often exhibit artifacts and drift; RadialAttention is more stable but still degrades in challenging scenes.}

	\label{fig:qualitative}
\end{figure*}

\noindent\textbf{Datasets.}
Consistent with~\cite{SVG1,SVG2}, we evaluate on \textbf{LongVBench}, a long-horizon benchmark based on VBench prompts (Penguin Benchmark with the released prompt optimization), designed to stress long-context generation. We also evaluate on \textbf{LongVGenBench}~\cite{gao2025longvie2}, a long-video generation benchmark that measure long-range quality under extended generation.

\noindent\textbf{Baselines.}
We compare our approach with:
\textbf{(i) Dense attention}, implemented using FlashAttention-3~\cite{flashattenion3} as our primary exact-attention baseline;
\textbf{(ii) KV-cache/computation reuse methods}, including TeaCache~\cite{TeaCache}, a training-free timestep-aware caching strategy that reuses intermediate computations when consecutive denoising steps are similar, and FlowCache~\cite{FlowCache}, a caching framework tailored to autoregressive video generation with chunk-wise policies and importance-based cache control; and
\textbf{(iii) Sparse attention methods}, including SVG~\cite{SVG1}, SVG2~\cite{SVG2}, and RadialAttn~\cite{li2025radial}, which we adapt and implement for autoregressive video diffusion models to enable a direct comparison in the streaming setting. 

For all baselines, we perform a full grid search over the reported hyperparameters to obtain the best performance in our setting.
In the supplementary material, we further report results on the exact same datasets, settings, and evaluation tables used in the corresponding baseline papers (for both short video generation and KV-cache compression).

\noindent\textbf{Implementation details.} 
All experiments are run on a single H100 GPU. To measure long-horizon scaling, we generate streams of 3000 frames \emph{without} imposing a context-window bound on the KV cache for any method. Following prior work~\cite{SVG1}, sparse attention is disabled for the first 30\% of denoising steps, and sparsification is applied only to blocks that empirically exhibit sparsity (about 70\% of blocks). We use FAISS~\cite{faiss} for LSH and quantization-based ANN retrieval, and FlashInfer~\cite{ye2025flashinfer} kernels for sparse attention. Unless stated otherwise, quantization uses 8-bit (ablations in Supp.). Qualitative figures show one representative frame per video.

\section{Results}
\subsection{Quantitative Results}

Table~\ref{tab:rf_longvbench_compact} compares our method with Dense FlashAttention-3 and current SoTA baselines on Rolling-Forcing LongVBench. It shows that our approach achieves the best quality–efficiency trade-off across all settings: both \textbf{TempCache-LSH/Quant} compress the KV cache aggressively (down to $\sim$16\% Min Density) while preserving high attention recall ($\sim$90--91\%) and matching dense quality (VBench $\approx$84.1). In contrast, TeaCache and FlowCache provide only modest speedups ($\times$1.1--$\times$1.3) while retaining most of the dense computation. For attention sparsification, offline-designed baselines (SVG1/2, RadialAttn) substantially degrade quality and/or recall, whereas our ANN-based self- and cross-attention pruning maintains near-dense quality at low density with consistently high recall. Finally, combining cache compression with SA/CA sparsity yields the strongest overall gains: \textbf{All Ours-LSH/Quant} achieves up to \textbf{$\times$10.7--$\times$10.8} end-to-end speedup while preserving FA3-level quality, while the best baseline combination (FlowCache+RadialAttn) reaches only $\times$4.4 and collapses in quality.
Results for generative world-model (LongVie2) in Supp. 

Figure~\ref{fig:time_memory_scaling} shows how Rolling-Forcing~\cite{liu2025rolling} scales with generation length.
\textbf{(a) Throughput.} As the KV cache grows, the dense attention baseline (FA3) slows down sharply, with frame-per-second (FPS) dropping continuously over the 3K-frame rollout. Existing sparsification baselines (SVG1/2, RadialAttention) also fail to sustain throughput in this long-horizon regime, largely due to substantial per-block preprocessing (e.g., clustering or energy/decay estimation) that must be repeated across transformer blocks, diffusion timesteps, and frames. In contrast, our method maintains constant FPS throughout generation: SA sparsity prevents attention cost from increasing with context length, and our ANN matching remains lightweight and kernel-friendly at scale. 
\textbf{(b) Peak memory.} Peak GPU memory for the dense attention baseline and current approaches increases with the expanding KV cache, whereas our method stays constant.
The same trend is also observed for world models; results on LongVie2 are provided in the supplementary material.

\subsection{Qualitative Results}
Figure~\ref{fig:qualitative} provides qualitative comparisons on long video generation with Rolling-Forcing. Across diverse prompts (characters, landscapes, fast motion, and complex textures), our method preserves the key visual attributes of Dense FlashAttention-3, including subject identity, fine details, and coherent motion, while sustaining stable appearance over time. In contrast, baseline sparse attention methods adapted from offline diffusion (e.g., SVG1/2) frequently introduce severe artifacts and temporal drift, leading to blurred structure, washed-out textures, or identity collapse as generation progresses. RadialAttention is more stable than SVG-style sparsity but still shows noticeable degradation relative to dense attention in challenging scenes.  More qualitative results can be found in the appendix.

\section{Summary}

In this paper, we study the attention in autoregressive video diffusion, where KV-cache growth makes streaming generation slower and more memory-intensive over time. We identify three redundancies: duplicate keys across frames, slowly evolving semantic Q/K, and expensive cross-attention over long prompts. We introduce a unified, training-free framework: \textbf{TempCache} (temporal KV merging), \textbf{AnnCA} (ANN-based prompt pruning), and \textbf{AnnSA} (ANN-based sparse self-attention). Our method is plug-and-play and requires no retraining or fine-tuning. Experiments show up to \textbf{$\times 5$--$\times 10$} end-to-end speedups with constant GPU memory over long video generations while preserving video quality.

\section*{Impact Statement}
This paper presents work whose goal is to advance the field
of Machine Learning. There are many potential societal
consequences of our work, none which we feel must be
specifically highlighted here.

\bibliography{main}
\bibliographystyle{icml2026}

\newpage
\appendix
\onecolumn

\section{Proof of Lemma~\ref{lem:redundancy_free_attention}}

\begin{lemma}[Redundancy-free attention]
\label{lem:redundancy_free_attention_supp}
Let $q \in \mathbb{R}^{d_k}$ be a query, and let
$K = (k_1, \dots, k_n)^\top \in \mathbb{R}^{n \times d_k}$ and
$V = (v_1, \dots, v_n)^\top \in \mathbb{R}^{n \times d_v}$
denote keys and values. Consider standard scaled dot-product attention
\begin{equation}
\label{eq:standard-attn-supp}
\mathrm{Attn}(q,K,V)
= \frac{\sum_{i=1}^{n} e^{s_i} v_i}{\sum_{i=1}^{n} e^{s_i}},
\qquad
s_i = \frac{q^\top k_i}{\sqrt{d_k}}.
\end{equation}
Assume the index set $\{1,\dots,n\}$ is partitioned into $g$ disjoint groups
$\{G_t\}_{t=1}^g$ such that all keys inside each group are identical, i.e.,
\begin{equation}
\label{eq:group-keys-identical-supp}
k_i = k'_t \quad \forall\, i \in G_t,
\end{equation}
for some representative key $k'_t \in \mathbb{R}^{d_k}$.
Define the group size (multiplicity) $m_t = |G_t|$ and the mean value
\begin{equation}
\label{eq:group-mean-value-supp}
\tilde{v}_t \;=\; \frac{1}{m_t}\sum_{i\in G_t} v_i.
\end{equation}
Let $s_t = \frac{q^\top k'_t}{\sqrt{d_k}}$ and $\tilde{s}_t = s_t + \log m_t$.
Then
\begin{equation}
\label{eq:grouped-attn-supp}
\mathrm{Attn}(q,K,V)
= \sum_{t=1}^{g} \frac{e^{\tilde{s}_t}}{\sum_{u=1}^{g} e^{\tilde{s}_u}} \,\tilde{v}_t.
\end{equation}
Equivalently, attention over $(q,K,V)$ is exactly equal to attention over the
$g$ representative keys $K'=(k'_1,\dots,k'_g)$ and mean values
$\tilde{V}=(\tilde{v}_1,\dots,\tilde{v}_g)$, with an additive bias $\log m_t$
applied to the logits.
\end{lemma}

\begin{proof}
Because all keys within a group $G_t$ are identical by~\eqref{eq:group-keys-identical-supp},
they induce the same attention score. For any $i \in G_t$,
\begin{equation}
s_i
= \frac{q^\top k_i}{\sqrt{d_k}}
= \frac{q^\top k'_t}{\sqrt{d_k}}
= s_t.
\label{eq:same-score-within-group-supp}
\end{equation}
We now regroup the numerator and denominator of~\eqref{eq:standard-attn-supp} by groups.

\paragraph{Regrouping the numerator.}
\begin{align}
\sum_{i=1}^{n} e^{s_i} v_i
&= \sum_{t=1}^{g} \sum_{i\in G_t} e^{s_i} v_i
 = \sum_{t=1}^{g} \sum_{i\in G_t} e^{s_t} v_i
= \sum_{t=1}^{g} e^{s_t} \sum_{i\in G_t} v_i.
\label{eq:numerator-regroup-supp}
\end{align}

\paragraph{Regrouping the denominator.}
\begin{align}
\sum_{i=1}^{n} e^{s_i}
&= \sum_{t=1}^{g} \sum_{i\in G_t} e^{s_i}
 = \sum_{t=1}^{g} \sum_{i\in G_t} e^{s_t}
 = \sum_{t=1}^{g} m_t e^{s_t}.
\label{eq:denominator-regroup-supp}
\end{align}

Substituting~\eqref{eq:numerator-regroup-supp} and~\eqref{eq:denominator-regroup-supp}
into~\eqref{eq:standard-attn-supp} yields
\begin{equation}
\mathrm{Attn}(q,K,V)
= \frac{\sum_{t=1}^{g} e^{s_t}\sum_{i\in G_t} v_i}{\sum_{t=1}^{g} m_t e^{s_t}}.
\label{eq:attn-after-regroup-supp}
\end{equation}
Next, observe that $m_t e^{s_t} = e^{s_t + \log m_t}$, and rewrite the numerator
to expose the group mean values:
\begin{align}
\mathrm{Attn}(q,K,V)
&= \frac{\sum_{t=1}^{g} e^{s_t}\sum_{i\in G_t} v_i}{\sum_{t=1}^{g} m_t e^{s_t}}
= \frac{\sum_{t=1}^{g} \left(m_t e^{s_t}\right)\left(\frac{1}{m_t}\sum_{i\in G_t} v_i\right)}
{\sum_{t=1}^{g} m_t e^{s_t}} \nonumber\\
&= \frac{\sum_{t=1}^{g} e^{s_t + \log m_t}\,\tilde{v}_t}{\sum_{t=1}^{g} e^{s_t + \log m_t}}.
\label{eq:attn-softmax-form-supp}
\end{align}
Defining $\tilde{s}_t = s_t + \log m_t$ turns~\eqref{eq:attn-softmax-form-supp} into a softmax
over groups:
\begin{equation}
\mathrm{Attn}(q,K,V)
= \sum_{t=1}^{g} \frac{e^{\tilde{s}_t}}{\sum_{u=1}^{g} e^{\tilde{s}_u}}\,\tilde{v}_t,
\end{equation}
which is exactly~\eqref{eq:grouped-attn-supp}.
\end{proof}

Lemma~\ref{lem:redundancy_free_attention_supp} implies that if multiple positions
share identical keys, they can be merged without any approximation error by
(i) averaging their values and (ii) adding a $\log m_t$ bias to the corresponding
logit. When no duplicate keys exist, $m_t=1$ for all $t$ and the formula reduces
to standard attention.

\begin{table*}[t]
\centering
\small
\setlength{\tabcolsep}{3.6pt}
\renewcommand{\arraystretch}{1.05}
\caption{\textbf{LongVie2 (LongVGenBench).} Quality--efficiency trade-off vs.\ dense FA3 across KV compression, sparse SA, sparse CA, and the full system. TempCache (LSH/Quant) compresses the KV cache to $\sim$33.1\% Min Density with high recall, ANN-based SA/CA preserve near-baseline quality at low density, and the full method achieves \textbf{$\times$6.3--$\times$6.9} speedup.}
\label{tab:longvie2_longvgen_compact}
\resizebox{\textwidth}{!}{%
\begin{tabular}{llccccccc}
\toprule
\textbf{LongVie2} & \textbf{Method} &
\textbf{PSNR}$\uparrow$ & \textbf{SSIM}$\uparrow$ & \textbf{LPIPS}$\downarrow$ &
\textbf{LongVGenBench}$\uparrow$ & \textbf{\shortstack{Min\\Density $\downarrow$}} &
\textbf{\shortstack{Max\\Recall $\uparrow$}} &
\textbf{\shortstack{Total \\ Speed $\uparrow$}} \\
\midrule
& FlashAttention 3 & -- & -- & -- & 69.67 & 100\% & 100\% & $\times$1.0 \\
\midrule
\multicolumn{9}{l}{\textbf{KV compression}} \\
& TeaCache & 13.6 & 0.254 & 0.399 & 61.41 & 92.0\% & 85.3\% & $\times$1.1 \\
& FlowCache & 20.6 & 0.315 & 0.357 & 62.78 & 90.4\% & 87.6\% & $\times$1.1 \\
& \textbf{TempCache-LSH (ours)} & \textbf{23.0} & \textbf{0.523} & \textbf{0.243} & \textbf{63.10} & \textbf{33.4\%} & \textbf{91.3\%} & \textbf{$\times$3.5} \\
& \textbf{TempCache-Quant (ours)} & \textbf{23.4} & \textbf{0.531} & \textbf{0.223} & \textbf{63.16} & \textbf{33.1\%} & \textbf{92.4\%} & \textbf{$\times$3.7} \\
\midrule
\multicolumn{9}{l}{\textbf{Sparse SA}} \\
& SVG1 & 11.6 & 0.557 & 0.831 & 24.83 & 78.3\% & 11.8\% & $\times$0.3 \\
& SVG2 & 12.7 & 0.185 & 0.812 & 25.65 & 70.9\% & 15.9\% & $\times$0.5 \\
& RadialAttn & 17.2 & 0.197 & 0.644 & 41.11 & 82.4\% & 44.7\% & $\times$2.7 \\
& \textbf{AnnSA-LSH (ours)} & \textbf{23.9} & \textbf{0.601} & \textbf{0.222} & \textbf{65.32} & \textbf{40.3\%} & \textbf{89.5\%} & \textbf{$\times$5.3} \\
& \textbf{AnnSA-Quant (ours)} & \textbf{23.9} & \textbf{0.633} & \textbf{0.199} & \textbf{66.43} & \textbf{44.9\%} & \textbf{87.0\%} & \textbf{$\times$5.4} \\
\midrule
\multicolumn{9}{l}{\textbf{Sparse CA}} \\
& \textbf{AnnCA-LSH (ours)} & \textbf{24.6} & \textbf{0.634} & \textbf{0.197} & \textbf{66.67} & \textbf{42.8\%} & \textbf{90.0\%} & \textbf{$\times$1.7} \\
& \textbf{AnnCA-Quant (ours)} & \textbf{23.3} & \textbf{0.620} & \textbf{0.203} & \textbf{63.77} & \textbf{44.9\%} & \textbf{85.6\%} & \textbf{$\times$1.8} \\
\midrule
\multicolumn{9}{l}{\textbf{Full}} \\
& FlowCache+RadialAttn & 18.6 & 0.247 & 0.572 & 49.84 & -- & -- & $\times$3.1 \\
& \textbf{All Ours-LSH} & \textbf{23.5} & \textbf{0.611} & \textbf{0.216} & \textbf{64.91} & -- & -- & \textbf{$\times$6.3} \\
& \textbf{All Ours-Quant} & \textbf{23.5} & \textbf{0.613} & \textbf{0.218} & \textbf{63.69} & \textbf{-} & \textbf{-} & \textbf{$\times$6.9} \\
\bottomrule
\end{tabular}%
}
\end{table*}

\begin{figure*}[t!]
	\centering
	\includegraphics[width=\linewidth]{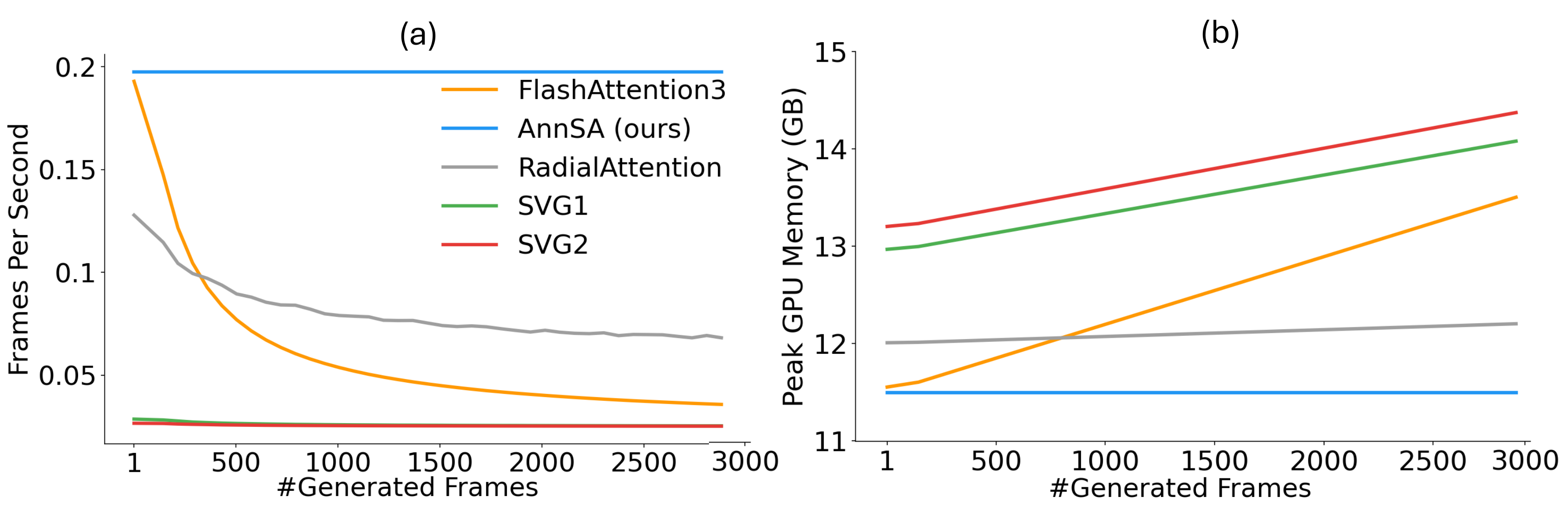}
	\caption{\textbf{Scaling with generation length in LonVie2.}
(a) \emph{Throughput.} As the context grows, dense FA3 slows down, and existing sparsification baselines (SVG1/2, RadialAttention) do not sustain throughput due to substantial per-block preprocessing repeated across transformer blocks, diffusion timesteps, and frames. Our method keeps FPS nearly constant over a 3K-frame rollout, making attention cost effectively independent of cache length.
(b) \emph{Peak memory.} FA3 and baselines show rising GPU memory with the expanding KV cache, whereas our memory stays flat, consistent with a bounded cache.}

	\label{fig:time_memory_scaling_supp}
\end{figure*}

\section{Additional Results}

\subsection{Quantitative Results}
Table~\ref{tab:longvie2_longvgen_compact} reports quantitative results on LongVGenBench with the LongVie2 world model. It shows that our components consistently improve the quality–efficiency trade-off relative to dense FlashAttention-3 and prior baselines. For \textbf{KV compression}, TeaCache/FlowCache provide only $\times$1.1 speedup while retaining $\sim$90\%+ density, whereas TempCache (LSH/Quant) reduces density to $\approx$33\% with high recall (91.3–92.4\%) and improves LongVGenBench (63.10–63.16) at $\times$3.5–$\times$3.7 speedup. For \textbf{sparse self-attention}, offline-designed sparsification (SVG1/2, RadialAttn) severely degrades LongVGenBench (24.83–41.11) and recall, while AnnSA (LSH/Quant) preserves much stronger quality (65.32–66.43) with high recall (87.0–89.5\%) and achieves $\times$5.3–$\times$5.4 speedup. For \textbf{sparse cross-attention}, AnnCA maintains high recall (85.6–90.0\%) and strong quality (63.77–66.67) at $\times$1.7–$\times$1.8 speedup. Finally, the \textbf{full system} delivers the best overall performance: All Ours achieves $\times$6.3–$\times$6.9 speedup wi  th substantially higher LongVGenBench (63.69–64.91) than the strongest baseline combination FlowCache+RadialAttn (49.84), indicating that our method scales effectively to world-model generation without sacrificing quality.

\begin{table}[t]
\centering
\small
\setlength{\tabcolsep}{5pt}
\renewcommand{\arraystretch}{1.15}
\caption{\textbf{TempCache-Quant on MAGI-1 and SkyReels-V2.} Results follow exactly the evaluation protocol of~\cite{FlowCache}. TempCache-Quant achieves the best efficiency--quality trade-off across both models, reducing PFLOPs and latency while maintaining or improving perceptual quality (VBench/LPIPS/SSIM/PSNR).} 

\label{tab:magi_skyreels}
\begin{tabular}{llccccccc}
\toprule
\textbf{Model} & \textbf{Method} &
\textbf{PFLOPs}$\downarrow$ &
\textbf{Speedup}$\uparrow$ &
\textbf{Latency (s)}$\downarrow$ &
\textbf{VBench}$\uparrow$ &
\textbf{LPIPS}$\downarrow$ &
\textbf{SSIM}$\uparrow$ &
\textbf{PSNR}$\uparrow$ \\
\midrule
\multirow{6}{*}{\textbf{MAGI-1}}
& Vanilla         & 306 & 1$\times$    & 2873 & 77.06\% & --     & --     & -- \\
& TeaCache-slow   & 294 & 1.12$\times$ & 2579 & 77.50\% & 0.6211 & 0.2801 & 13.26 \\
& TeaCache-fast   & 225 & 1.44$\times$ & 1998 & 70.11\% & 0.8160 & 0.1138 &  8.94 \\
& FlowCache-slow  & 161 & 1.86$\times$ & 1546 & 78.96\% & 0.3160 & 0.6497 & 22.34 \\
& FlowCache-fast  & 140 & 2.38$\times$ & 1209 & 77.93\% & 0.4311 & 0.5140 & 19.27 \\
& \textbf{TempCache-Quant (ours)} & \textbf{110} & \textbf{4.11$\times$} & \textbf{1009} & \textbf{78.99\%} & \textbf{0.3156} & \textbf{0.6555} & \textbf{23.12} \\
\midrule
\multirow{6}{*}{\textbf{SkyReels-V2}}
& Vanilla         & 113 & 1$\times$    & 1540 & 83.84\% & --     & --     & -- \\
& TeaCache-slow   &  58 & 1.89$\times$ &  814 & 82.67\% & 0.1472 & 0.7501 & 21.96 \\
& TeaCache-fast   &  49 & 2.2$\times$  &  686 & 80.06\% & 0.3063 & 0.6121 & 18.39 \\
& FlowCache-slow  &  36 & 5.88$\times$ &  262 & 83.12\% & 0.1225 & 0.7890 & 23.74 \\
& FlowCache-fast  &  28 & 6.7$\times$  &  230 & 83.05\% & 0.1467 & 0.7635 & 22.95 \\
& \textbf{TempCache-Quant (ours)} & \textbf{11} & \textbf{9.25$\times$} & \textbf{146} & \textbf{83.82\%} & \textbf{0.1160} & \textbf{0.8020} & \textbf{23.99} \\
\bottomrule
\end{tabular}
\end{table}

\begin{table*}[t]
\centering
\small
\setlength{\tabcolsep}{5.0pt}
\renewcommand{\arraystretch}{1.15}
\caption{Results on \textbf{5-second} video generation for HunyuanVideo (117 frames) and Wan2.1-14B (69 frames). AnnSA (LSH/Quant) matches or slightly improves the quality of prior sparse-attention baselines  while staying in a similar efficiency regime. On these short clips, the lowest latency / highest speedup is still achieved by STA (FA3), since sparse-attention kernels introduce extra overhead and are typically less optimized than FlashAttention at short sequence lengths. Our main benefits appear in long-horizon generation where attention and KV-cache growth dominate compute and memory.}
\label{tab:hunyuan_wan_annsa}
\begin{tabular}{llcccccccc}
\toprule
\textbf{Model} & \textbf{Method} &
\textbf{PSNR}$\uparrow$ & \textbf{SSIM}$\uparrow$ & \textbf{LPIPS}$\downarrow$ &
\textbf{Vision Reward}$\uparrow$ &
\textbf{PFLOPs}$\downarrow$ & \textbf{Latency (s)}$\downarrow$ & \textbf{Speedup}$\uparrow$ \\
\midrule
\textbf{HunyuanVideo} \\ \textbf{(117 frames)}
& Original     & --   & --    & --    & 0.141 & 612 & 1649 & -- \\
& STA (FA3)    & 26.7 & 0.866 & 0.167 & 0.132 & 331 & \textbf{719} & \textbf{2.29$\times$} \\
& PA           & 22.1 & 0.764 & 0.256 & 0.140 & 339 & 1002 & 1.65$\times$ \\
& SVG          & 27.2 & 0.895 & 0.114 & 0.144 & 340 &  867 & 1.90$\times$ \\
& RadianAttn   & 27.3 & 0.886 & 0.114 & 0.139 & 339 &  876 & 1.88$\times$ \\
& AnnSA-LSH (ours)   & 27.2 & 0.899 & 0.112 & 0.149 & \textbf{322} &  899 & 1.77$\times$ \\
& AnnSA-Quant (ours) & \textbf{27.4} & \textbf{0.900} & \textbf{0.111} & \textbf{0.150} & 341 &  912 & 1.81$\times$ \\
\midrule
\textbf{Wan2.1-14B} \\ \textbf{(69 frames)}
& Original     & --   & --    & --    & \textbf{0.136} & 560 & 1630 & -- \\
& STA (FA3)    & 22.9 & 0.830 & 0.171 & 0.132 & \textbf{322} & \textbf{812} & \textbf{2.01$\times$} \\
& PA           & 22.4 & 0.790 & 0.176 & 0.126 & 324 &  978 & 1.67$\times$ \\
& SVG          & 23.2 & 0.825 & 0.202 & 0.114 & 324 &  949 & 1.71$\times$ \\
& RadianAttn   & 23.9 & 0.842 & 0.163 & 0.128 & 323 &  917 & 1.77$\times$ \\
& AnnSA-LSH (ours)   & 23.5 & \textbf{0.844} & 0.161 & 0.131 & 354 &  987 & 1.71$\times$ \\
& AnnSA-Quant (ours) & \textbf{27.4} & 0.839 & \textbf{0.159} & 0.133 & 366 &  912 & 1.66$\times$ \\
\bottomrule
\end{tabular}
\end{table*}

\begin{figure*}[t!]
	\centering
	\includegraphics[width=\linewidth]{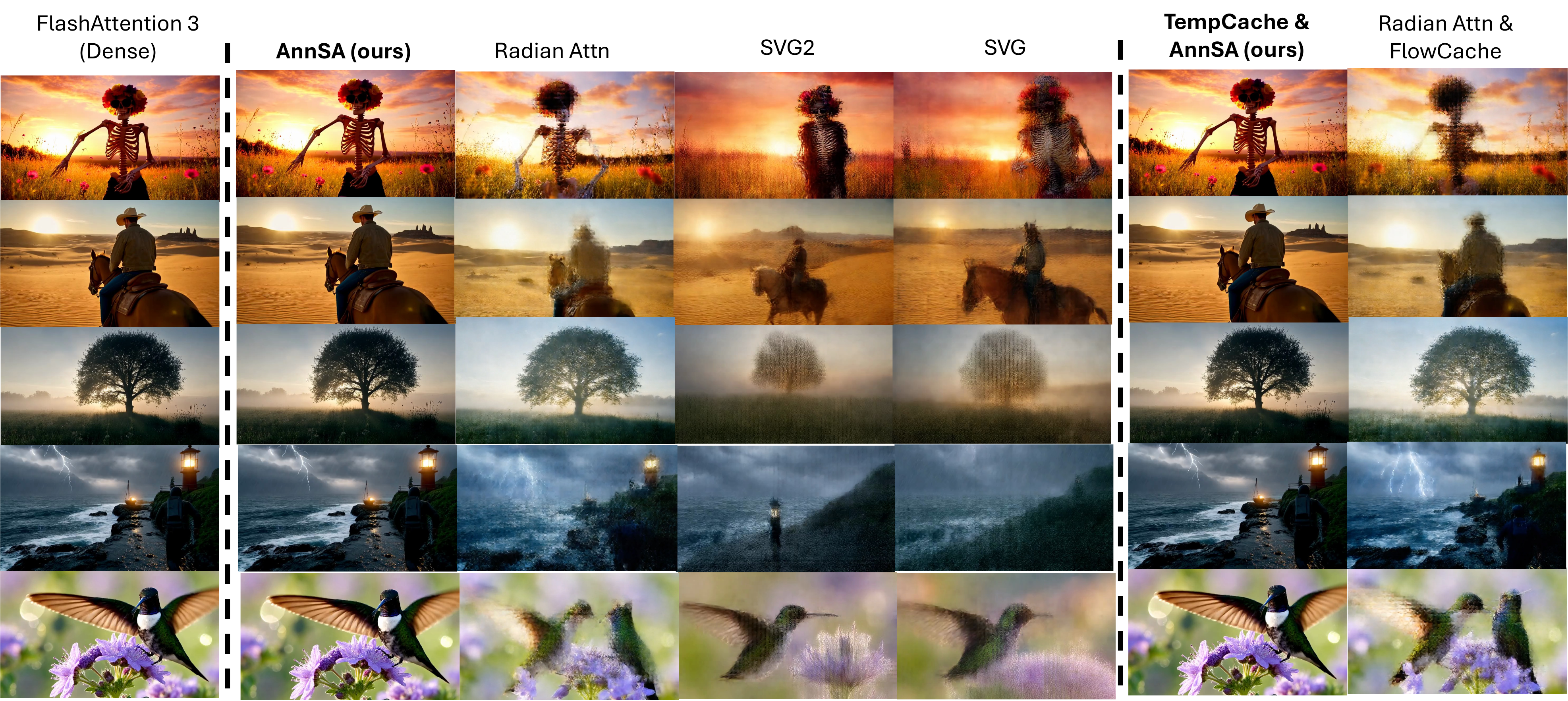}
	\caption{\textbf{Qualitative results on video world-model  LongVie2.} Our TempCache+ANN sparsification preserves the visual fidelity and temporal consistency of dense FlashAttention-3 across diverse prompts, while offline-designed sparsification baselines (SVG1/2) often exhibit artifacts and drift; RadialAttention is more stable but still degrades in challenging scenes. full videos and additional results can be found in the supplementary material. }
	\label{fig:longvie2_time_memory_scaling_supp}
\end{figure*}

Figure~\ref{fig:time_memory_scaling_supp} shows attention scaling on LongVie2 as generation length increases.
\textbf{(a) Attention time.} Dense FlashAttention-3 exhibits steadily increasing attention latency with more frames, and RadialAttention remains slower than our method; SVG1/2 are orders-of-magnitude slower (log-scale). In contrast, our approach keeps attention time constant across the full 3K-frame generation.
\textbf{(b) Peak memory.} Peak GPU memory for FA3 (and other baselines) grows with context length due to KV-cache expansion, whereas our method stays essentially flat throughout generation, indicating a constant-size KV cache. Together, these results confirm that our approach avoids both latency and memory growth for long-horizon world-model generation.

Table~\ref{tab:magi_skyreels} further reports results on MAGI-1 and SkyReels-V2, following \emph{exactly the same evaluation protocol} as~\cite{FlowCache}. Across both models, TempCache-Quant (ours) achieves the best efficiency--quality trade-off: it reduces compute (PFLOPs) and latency while maintaining or improving generation quality. On MAGI-1, TempCache-Quant lowers compute to 110 PFLOPs and improves speed to 4.11$\times$ with higher VBench (78.99\%) and strong perceptual fidelity (LPIPS 0.3156, SSIM 0.6555, PSNR 23.12). On SkyReels-V2, it further scales to 9.25$\times$ speedup with the lowest compute (11 PFLOPs) and latency (146s), while preserving near-baseline VBench (83.82\% vs.\ 83.84\%). These results demonstrate that TempCache provides substantial acceleration without compromising visual quality, outperforming prior caching baselines under a fair, identical setup.

Table~\ref{tab:hunyuan_wan_annsa} reports \textbf{5-second} generation results on HunyuanVideo and Wan2.1-14B.
Across both models, \textbf{AnnSA-LSH/Quant (ours)} achieves quality comparable to (and in several metrics slightly better than) existing sparse-attention baselines, while remaining in the same efficiency regime.
Notably, the best latency/speedup is still attained by \textbf{STA (FA3)} on these short clips, since sparse attention relies on kernels that incur additional overhead and are typically less optimized than FlashAttention at short sequence lengths.
These results indicate that our method is already competitive for short videos, while our main gains emerge in long-horizon generation where attention and KV-cache growth dominate runtime and memory (cf. long-context experiments).

\subsection{Qualitative Results}

Figure~\ref{fig:longvie2_time_memory_scaling_supp} shows additional qualitative comparisons on \textbf{LongVie2} rollouts. Our method remains visually close to dense FlashAttention-3, preserving scene layout, subject identity, and lighting over time, even in challenging cases with thin structures and low-contrast textures. In contrast, offline-designed sparsification baselines (SVG1/2) often accumulate artifacts during long-horizon generation, leading to blur, texture wash-out, and partial disappearance or distortion of the main subject. RadialAttention is generally more stable than SVG-style sparsity but still exhibits noticeable degradation in several examples. Overall, these LongVie2 results further support that our KV compression together with ANN-based SA/CA sparsity maintains temporal coherence and perceptual quality under long rollouts while enabling efficient generation.

\section{Ablation Study}

\begin{figure*}[t!]
	\centering
	\includegraphics[width=\linewidth]{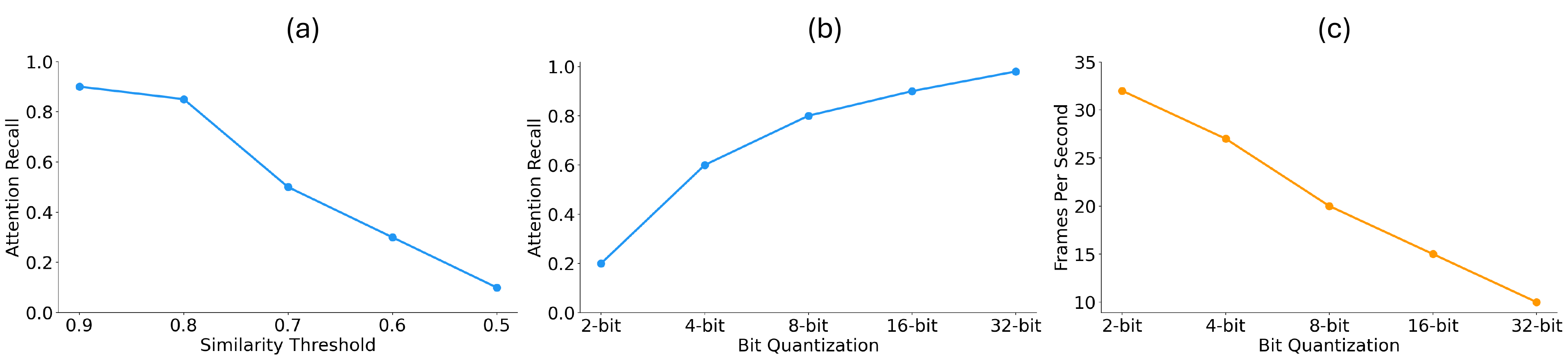}
	\caption{\textbf{Ablations on TempCache and quantized ANN.}
(a) TempCache KV compression trades accuracy for compression: lowering the key-similarity threshold increases merging but reduces attention recall.
(b) Quantization bit-width vs.\ recall: higher precision improves ANN matching quality.
(c) Quantization bit-width vs.\ throughput: lower precision yields higher FPS, highlighting the accuracy--speed trade-off.}

	\label{fig:ablation}
\end{figure*}

\paragraph{TempCache similarity threshold.}
Figure~\ref{fig:ablation}(a) analyzes the key design trade-offs in TempCache. TempCache merges temporally corresponding keys when their similarity exceeds a threshold. As the threshold decreases from $0.9$ to $0.5$, the compression becomes more aggressive, but attention recall drops from $0.90$ to $0.10$, indicating that over-merging can remove keys that still contribute meaningfully to attention. This motivates using conservative thresholds that capture near-duplicate temporal features while avoiding excessive approximation.

\paragraph{Quantization precision for ANN matching.}
Figure~\ref{fig:ablation} analyzes the key design trade-offs our quantization-based ANN matching.
We vary the bit-width used for ANN retrieval. Increasing precision improves matching quality, raising recall from $0.20$ (2-bit) to $0.98$ (32-bit), while also reducing throughput as computation and memory bandwidth increase (FPS decreases from $32$ to $10$ from 2-bit to 32-bit). Overall, we observe a clear accuracy--efficiency trade-off, where mid-precision settings (e.g., 8-bit) provide a strong balance, achieving high recall ($\approx 0.80$) with substantial speed gains.

\paragraph{Representative selection in KV merging.}
When grouping temporally corresponding keys, TempCache must choose a single representative per group. We compare (i) \emph{Last-key} (ours), which keeps the most recent key, (ii) \emph{Mean-key}, which averages keys in the group, and (iii) \emph{Medoid-key}, which selects the most central key. Last-key performs best, achieving 90\% attention recall versus 75\% for mean-key and 66\% for medoid-key. This aligns with autoregressive generation: current-frame queries are most compatible with recent context, whereas averaging or selecting a central past key can blur recency-specific features and distort the attention distribution.



\end{document}